\documentclass{article}

\usepackage{arxiv}

\usepackage[utf8]{inputenc} 
\usepackage[T1]{fontenc}    
\usepackage{hyperref}       
\usepackage{url}            
\usepackage{booktabs}       
\usepackage{amsfonts}       
\usepackage{nicefrac}       
\usepackage{microtype}      
\usepackage{lipsum}

\usepackage{cite}
\usepackage{amsmath,amssymb,amsfonts}
\usepackage{graphicx}
\usepackage{algorithm}
\usepackage[noend]{algpseudocode}
\usepackage{hyperref}
\usepackage{textcomp}
\usepackage{threeparttablex} 

\usepackage{mathtools}
\usepackage{caption}
\usepackage{float}
\usepackage{stmaryrd}
\usepackage{epstopdf}
\usepackage{multirow}
\usepackage{pifont}
\usepackage{caption}
\usepackage{subcaption}
\usepackage{xcolor}

\newcommand{\R}{{\mathbb{R}}}

\newcommand{\B}{{\mathbb B}}

\newcommand{\N}{{\mathbb{N}}}

\newcommand{\X}{{\mathbf{X}}}

\newcommand{\T}{{\mathbf{T}}}
\newcommand{\So}{{\mathbf{S}}}

\newcommand{\U}{{\mathbf{U}}}

\newcommand{\cen}{\mathsf c}
\newcommand{\rad}{\mathsf r}
\newcommand{\A}{{\mathbf{A}}}

\newtheorem{theorem}{Theorem}[section]

\newtheorem{assumption}{Assumption}

\newtheorem{definition}[theorem]{Definition}
\newtheorem{lemma}[theorem]{Lemma}
\newtheorem{remark}[theorem]{Remark}
\newtheorem{problem}[theorem]{Problem}
\newenvironment{proof}{\paragraph{Proof:}}{\hfill$\square$}

\title{Temporal Reach-Avoid-Stay Control for Differential Drive Systems via Spatiotemporal Tubes
\thanks{ This work was supported in part by the SERB Start-Up Research Grant; in part by the ARTPARK. The work of Ratnangshu Das was supported by the Prime Minister’s Research Fellowship from the Ministry of Education, Government of India.}
}

\author{
 Ratnangshu Das \\
  Robert Bosch Centre for Cyber-Physical Systems\\
  IISc, Bengaluru, India\\
  \texttt{ratnangshud@iisc.ac.in} \\
   \And
 Ahan Basu \\
  Robert Bosch Centre for Cyber-Physical Systems\\
  IISc, Bengaluru, India\\
  \texttt{ahanbasu@iisc.ac.in} \\
  \And
Christos Verginis \\
  Department of Electrical Engineering\\
  Uppsala University, Uppsala, Sweden\\
  \texttt{christos.verginis@angstrom.uu.se} \\
  \And
 Pushpak Jagtap \\
  Robert Bosch Centre for Cyber-Physical Systems\\
  IISc, Bengaluru, India\\
  \texttt{pushpak@iisc.ac.in} \\
}

\begin{document}
\maketitle

\begin{abstract}
This paper presents a computationally lightweight and robust control framework for differential-drive mobile robots with dynamic uncertainties and external disturbances, guaranteeing the satisfaction of Temporal Reach-Avoid-Stay (T-RAS) specifications. The approach employs circular spatiotemporal tubes (STTs), characterized by smoothly time-varying center and radius, to define dynamic safe corridors that guide the robot from the start region to the goal while avoiding obstacles. In particular, we first develop a sampling-based synthesis algorithm to construct a feasible STT that satisfies the prescribed timing and safety constraints with formal guarantees. To ensure that the robot remains confined within this tube, we then analytically design a closed-form control that is computationally efficient and robust to disturbances. The proposed framework is validated through simulation studies on a differential-drive robot and benchmarked against state-of-the-art methods, demonstrating superior robustness, accuracy, and computational efficiency.
\end{abstract}

\section{Introduction}

Differential-drive mobile robots are among the most widely used robotic platforms due to their mechanical simplicity, maneuverability, and effectiveness in real-world navigation and manipulation tasks. They have been extensively adopted across diverse domains, including industrial applications \cite{fragapane2022increasing}, the healthcare sector \cite{zhou2018fuzzy}, autonomous exploration \cite{ferreira2024autonomous}, 
and many other fields. A fundamental challenge in deploying mobile robots is ensuring that they can reach designated targets in a fixed time while avoiding time-varying obstacles and adhering to state constraints. Temporal-reach-avoid-stay (T-RAS) specifications are crucial for addressing such challenges \cite{Meng1}. T-RAS formulations also serve as building blocks for more complex temporal-logic tasks \cite{Kloetzer, das2025spatiotemporal}. Therefore, developing and implementing safe and reliable control strategies is essential to perform these tasks effectively.

A variety of control approaches have been developed to enforce reach-avoid-type specifications, including symbolic control, Control Barrier Functions (CBFs), and Model Predictive Control (MPC). Symbolic control techniques \cite{tabuada2009verification} rely on abstractions of the state and input spaces and have shown effectiveness in solving complex temporal-logic tasks. However, these methods often suffer from severe computational overhead as the dimensionality or nonlinearity of the system increases, making them difficult to scale. 

To overcome discretization-related limitations, Control Barrier Functions (CBFs) \cite{CBF} provide a continuous-state, optimization-based framework for synthesizing safety controllers. They have been successfully employed for dynamic obstacle avoidance \cite{C3BF,tayal2024collision} and for enforcing temporal-logic constraints \cite{CBF_STL, Meng3}. 
MPC \cite{sun2017disturbance, jian2022dynamic} has also been widely employed to enforce reach-avoid and safety-critical specifications by optimizing future trajectories subject to dynamic and constraint models \cite{MPC}.
Despite solving the scalability issue of symbolic methods, CBF- and MPC-based control still relies on an optimization step, which can be computationally demanding for high-dimensional or fast-evolving systems. Moreover, they struggle to guarantee satisfaction of time-critical objectives such as Temporal reachability, and their performance degrades under dynamic uncertainties and external disturbances, {which affect substantially real robots}.

{In contrast, funnel-based control \cite{PPC1} provides a natural framework capable of handling both external disturbances and time-critical objectives. It has been successfully applied to tracking control of unknown nonlinear systems \cite{mishra2023approximation} and multi-agent coordination tasks \cite{Funnel_STL_MAS}. However, the traditional funnel technique is not inherently equipped to address constraints such as obstacle avoidance.  
Some recent studies \cite{RB,vrohidis2018prescribed} integrate articifial potential-field with funnel-invariance methods to establish collision-free navigation, without, however, accounting for dynamic uncertainties. The recently introduced KDF method \cite{verginis2022kdf} combines sampling-based planning with funnel control to guarantee collision-free navigation for systems with uncertain dynamics, limited, however to fully actuated systems. To tackle that, KDF was recently extended in \cite{lapandic2024kinodynamic} to account for underactuated surface vessels with uncertain dynamics, still limited to static environments. Additionally, \cite{lapandic2024kinodynamic} restricts the control design to output only positive velocities, which, although suitable for surface vessels, can be conservative for general differential-drive mobile robots. }

To overcome the static-obstacle assumption, the spatiotemporal tube (STT) framework \cite{STT} was introduced. STT define smooth, time-varying regions in space that evolve dynamically to form safe corridors around obstacles, allowing system trajectories to satisfy T-RAS objectives \cite{das2025spatiotemporal}.  
In this work, we extend the Spatiotemporal Tube (STT) framework to a broader class of systems, specifically the underactuated systems, with the main focus lying on differential-drive mobile robots, to ensure satisfaction of T-RAS specification. To achieve this, we introduce circular STT, characterized by smooth, time-varying center and radius, which serve as safe, time-evolving corridors guiding the robot toward its goal while avoiding obstacles. The contributions of this paper are twofold. First, we present a sampling-based synthesis procedure to construct circular STT that originate within the start region, strictly avoid obstacles, and reach the target region within a prescribed time horizon. Second, we derive a closed-form control law that is inherently robust to external disturbances and guarantees the robot remains confined within the designed STT, thereby ensuring satisfaction of the T-RAS task.
The effectiveness of the proposed framework is demonstrated through simulation studies on a differential-drive robot, and its performance is benchmarked against state-of-the-art approaches, highlighting its accuracy, robustness, and computational efficiency.

\section{Preliminaries and Problem Formulation}
\label{sec:prelim}
\subsection{Notation}
For $a,b\in\N$ with $a\leq b$, we denote the closed interval in $\N$ as $[a;b] := \{a, a+1, \ldots, b\}$. A ball centered at $\cen \in \mathbb{R}^n$ with radius $\rad \in \mathbb{R}^+$ is defined as $\mathcal{B}(\cen, \rad) := \{ x \in \mathbb{R}^n \mid \|x - \cen\| \leq \rad \}$. The distance of a point $x\in \R^n$ from a set $A \subseteq \R^n$ is defined as $\text{dist}(x, A) := \inf_{y \in A} \|x - y\|.$ All other notation in this paper follows standard mathematical conventions.

\subsection{System Definition}
We consider the differential-drive mobile robot whose motion is controlled by the linear velocity $v \in \mathbb{R}$ and the angular velocity $\omega \in \mathbb{R}$. The system dynamics $\mathcal{S}$ is given by
\begin{equation} \label{eqn:sysdyn_unicycle}
\mathcal{S}: \dot{\xi} = f(\xi, u) + d(\xi,t) \implies
\begin{bmatrix}
\dot{x}_1 \\ \dot{x}_2 \\ \dot{\theta}
\end{bmatrix}
=
\begin{bmatrix}
v \cos(\theta) \\ v \sin(\theta) \\ \omega
\end{bmatrix} + 
d(\xi,t),
\end{equation}
where $ u = [v, \omega]^\top \in \R^2$ is the control input and $d(\xi,t) = [d_1(\xi,t), d_2(\xi,t), d_\theta(\xi,t)]^\top \in \R^3$ is an unknown term representing model uncertainties external time-varying disturbances that is locally Lipschitz in $\xi$ and uniformly bounded in $t$. The system state is defined as $ \xi = [x_1, x_2, \theta]^\top$, where $x = [x_1, x_2]^\top \in \R^2 $ is the position in a 2D plane and $ \theta \in \R $ the heading angle with respect to a global reference frame.

\subsection{System Specification}
Let $\X \subseteq \R^2$ denote the state space and $\U: \R_0^+ \rightarrow \X$ be the time-varying unsafe set, whose maximum permissible velocity is given by $\mathcal{L}_U$. In this work, the time evolution of the unsafe set is considered to be known a priori, but that does not limit the scope of the work. Such an assumption is standard in<§ the literature of reach-avoid problems \cite{ahmadzadeh2009multi, chen2015safe}.
The initial and target sets are denoted by $\So \subset \X \setminus \U(0)$ and $\T \subset \X  \setminus \U(t_c)$, respectively, both assumed to be compact and connected. $t_c \in \R^+$ is the prescribed time of task completion. The control objective is formalized as a Temporal Reach-Avoid-Stay (T-RAS) task, defined below.

\begin{definition}[Temporal Reach-Avoid-Stay Task]\label{def:ptras}
    Given a system $\mathcal{S}$, a state space $\X$, a time-varying unsafe set $\U(t)$, an initial set $\So$, a target set $\T$, and the prescribed-time $t_c$, the system satisfies the T-RAS specification if, for a given initial condition $x(0) \in \So$, there exists $t \in [0,t_c]$, such that $x(t) \in \T$ and for all $s \in [0,t_c], x(s) \in \X \setminus \U(s)$.
\end{definition}

We now state the control problem addressed in this work.

\begin{problem}\label{prob1}
   Given the differential-drive system $\mathcal{S}$ in \eqref{eqn:sysdyn_unicycle} and a T-RAS task in Definition~\eqref{def:ptras}, the goal is to design a closed form control law $[v,\omega]^\top$, that is robust to external disturbances and ensures the system satisfies the specified T-RAS task.
\end{problem}

\section{Spatiotemporal Tube (STT)}\label{sec:stt}
To satisfy the given T-RAS specification, we leverage Spatiotemporal Tube (STT). STT act as time-varying structures in the state space that can effectively capture the dynamic constraints of a T-RAS task. 

\begin{definition}[STT for T-RAS Specification]\label{def:stt}
Consider a T-RAS task in Definition~\ref{def:ptras}. A time-varying region 
$\Gamma(t) = \B\big(\cen(t),\rad(t)\big)$
characterized by continuously differentiable centre $\cen:\R_0^+\rightarrow\R^2$ and radius $\rad:\R_0^+\rightarrow\R^+$, is a valid STT for the T-RAS specification, if the following hold:
\begin{subequations}
\begin{align}\label{eqn:stt_stl}
    &\rad(t) > 0, \forall t \in [0, t_c], \\ 
    &\Gamma(0) \subseteq \So, \ \Gamma(t_c) \subseteq \T,\
    \Gamma(t) \subseteq \X \setminus \U(t), \forall t \in [0, t_c]. 
\end{align}
\end{subequations}
Thus, the radius remains strictly positive and the tube starting from the initial set $\So$, reaches the target set $\T$ at time $t_c$, while avoiding the unsafe set $\U$ and staying within the state space $\X$.
\end{definition}

If the vehicle position is constrained to lie within the STT:
\begin{align} \label{eqn:stt_constrain}
    x(t) \in \Gamma (t), \forall t \in [0, t_c],
\end{align}
the T-RAS specification is guaranteed to be satisfied.

We now propose a sampling-based technique to construct the STT that connects $\So$ to $\T$ in time $t_c$, while avoiding $\U$. 

We first fix the structure of the center and radius curves as:
\begin{align}
    \cen_i(q_{\cen,i},t) &= \sum_{k=1}^{z_i} q_{\cen,i}^{(k)} b_{\cen,i}^{(k)} (t), i \in [1;2], \\
    \rad(q_\rad,t) &= \sum_{k=1}^{z_\rad} q_\rad^{(k)} b_\rad^{(k)}(t),
\end{align}
where $\cen(q_{\cen},t) = [\cen_1(q_{\cen,1},t), \ldots, \cen_n(q_{\cen,n},t)]^\top$ gives the tube's centre, and $\rad(q_\rad,t)$ defines the tube's radius. $b_{\cen,i}(t)=[b_{\cen,i}^{(1)}, ..., b_{\cen,i}^{(z_i)}]^\top$, $b_\rad(t)=[b_{\rad}^{(1)}, ..., b_{\rad}^{(z_\rad)}]^\top$ are user-defined continuously differentiable basis functions and $q_\cen = [q_{\cen,1}, q_{\cen,2}]^\top$ with $q_{\cen,i} = [q_{\cen,i}^{(1)}, ..., q_{\cen,i}^{(z_i)}]^\top \in \mathbb{R}^{z_i}, q_\rad = [q_\rad^{(1)}, ..., q_\rad^{(z_\rad)}] \in \mathbb{R}^{z_\rad}$ denote the unknown coefficients.

To satisfy the conditions in Definition~\ref{def:stt}, we formulate the following Robust Optimization Program (ROP):
\begin{subequations} \label{eq:ROP}
\begin{align}
& \min_{[q_{\cen,1}, q_{\cen,2}, q_\rad,\eta]} \quad \eta, \qquad \textrm{s.t.} \notag \\
& \quad \| \cen(q_{\cen},0) - \cen_{\So} \| = 0, \ \rad(c_\rad,0)  = \rad_{\So}, \\
& \quad \| \cen(q_{\cen},t_c) - \cen_{\T} \| = 0, \ \rad(c_\rad,t_c)  = \rad_{\T}, \\
& \forall t \in [0,t_c]: \notag \\
& \quad \| \cen(q_{\cen},t) - \cen_{\X} \| + \rad(c_\rad,t) - \rad_{\X} \leq \eta, \\
& \quad -\rad(q_\rad, t) + \rad_d \leq \eta, \\
& \quad -\text{dist}(\cen(q_{\cen},t),\U(t)) + \rad(q_\rad,t) \leq \eta,
\end{align}
\end{subequations}
where $\B(\cen_\So, \rad_\So) \subseteq \So$, $\B(\cen_{\T}, \rad_\T) \subseteq \T$, and $\B(\cen_{\X}, \rad_\X) \subseteq \X$, 
with $\cen_\So \in \So, \cen_\T \in \T, \cen_\X \in \X$ and $\rad_\So, \rad_\T, \rad_\X \in \R^+.$ 
$\rad_{d} \in \R^+$ is a user-defined lower bound on the tube radius. 

Note that $\eta$ is introduced as a slack variable, which has to be minimized in order to reduce the worst-case violations of the constraints across the dimensions. One can readily observe that a solution to the ROP with $\eta^* \leq 0$ ensures that all conditions in Definition~\ref{def:stt} are satisfied. The proposed solution is sound: whenever a feasible solution exists, the resulting tubes guarantee that the robot will reach the goal, avoiding the obstacles. The feasibility of the optimization problem \eqref{eq:ROP}, however, depends on the degree of the polynomial basis functions used to parameterize the STTs. If no feasible solution is obtained, then a higher degree polynomial may be required to represent the tubes.

The ROP in \eqref{eq:ROP} inherently involves an infinite set of constraints defined over continuous time, making direct computation infeasible. To overcome this, we design a sampling-based framework for constructing the tube.

We sample $N$ discrete points $t_s$, from the continuous time space $[0, t_c]$, where $s = [1;N]$. Consider a time-ball $T_s $ around each sample $t_s$ with radius $\epsilon$. The samples are chosen such that for any time $t \in [0, t_c]$, a sampled point $t_s$ is sufficiently close
\begin{align}\label{eq:ball}
    | t - t_s | \leq \epsilon , \forall t \in [0, t_c].
\end{align}  
This ensures that the union of all these time-balls covers the entire domain: $\bigcup_{s=1}^{N} T_s \supset [0,t_c]$.

We then construct the associated Scenario Optimization Program (SOP) by replacing the continuous-time constraints with constraints at the sampled time instances $t_s$:
\begin{subequations} \label{eq:SOP}
\begin{align}
& \min_{[q_{\cen,1}, q_{\cen,2}, q_\rad,\eta]} \quad \eta, \qquad \textrm{s.t.} \notag \\
& \quad \| \cen(q_{\cen},0) - \cen_{\So} \| = 0, \ \rad(q_\rad,0)  = \rad_{\So}, \\
& \quad \| \cen(q_{\cen},t_c) - \cen_{\T} \| = 0, \ \rad(q_\rad,t_c)  = \rad_{\T}, \\
& \forall t_s \in [0,t_c], s \in [1;N]: \notag \\
& \quad \| \cen(q_{\cen},t_s) - \cen_{\X} \| + \rad(q_\rad,t_s) - \rad_{\X} \leq \eta, \\
& \quad -\rad(q_\rad, t_s) + \rad_d \leq \eta, \\
& \quad -\text{dist}(\cen(q_{\cen},t_s),\U(t)) + \rad(q_\rad,t_s) \leq \eta.
\end{align}
\end{subequations}

One can observe that SOP in \eqref{eq:SOP} has a finite number of constraints of the same form as \eqref{eq:ROP}. To guarantee that the Tube formed by solving the SOP in \eqref{eq:SOP}, fulfill the ROP constraints in \eqref{eq:ROP}, we assume the following: 
\begin{assumption} \label{assum:funlip}
    The functions $\cen(q_{\cen},t)$ and $\rad(q_{\rad},t)$ are Lipschitz continuous in $t$, with constants $\mathcal{L}_{\cen}$ and $\mathcal{L}_{\rad}$, respectively.
\end{assumption}

\begin{lemma}\label{lem:point-set-dist}
If the point-to-set distances of two points $y_1$ and $y_2$ from a set $\A$ is defined as $\text{dist}(y_1,\A)$ and $\text{dist}(y_2,\A)$, then $\text{dist}(y_1,\A) - \text{dist}(y_2,\A) \leq \| y_1 - y_2 \|$.
\end{lemma}

\begin{proof}
Let $a_1, a_2 \in \A, a_1 \neq a_2$ be the points corresponding to the optimal distances of the set from points $y_1, y_2$ respectively, \textit{i.e.,} $\text{dist}(y_1,\A) = \| y_1 - a_1 \|$ and $\text{dist}(y_2,\A) = \| y_2 - a_2 \|$. 
Then, $\text{dist}(y_1,\A) \leq \| y_1 - a_2 \| \leq \| y_1 - y_2 \| + \| y_2 - a_2 \| \leq \| y_1 - y_2 \| + \text{dist}(y_2,\A)$ (using triangle inequality). Hence, $\text{dist}(y_1,\A) - \text{dist}(y_2,\A) \leq \| y_1 - y_2 \|$.
\end{proof}

The following theorem proves that the Lipschitz continuity of the center and radius ensures that the solution obtained by solving the SOP generalizes to the entire continuous domain.

\begin{theorem} \label{th:constr}
     Let $[q_\cen^*, q_\rad^*, \eta^*]$ be the optimal solution of the SOP in \eqref{eq:SOP}. If the condition 
    \begin{equation} \label{eq:satisfy}
        \eta^* + \mathcal{L}\epsilon \leq 0,
    \end{equation}
    holds, where $\mathcal{L} = \mathcal{L}_{\cen} + \mathcal{L}_U + \mathcal{L}_{\rad}$, then the resulting STT
    $$\Gamma(q_\cen^*, q_\rad^*, t) = \B(\cen(q_\cen^*, t), \rad(q_\rad^*, t)),$$
    satisfies all the conditions of Definition \ref{def:stt}. 
\end{theorem}

\begin{proof}
This proof shows that if condition \eqref{eq:satisfy} is met, the STT, $\Gamma(q_\cen^*, q_\rad^*, t)$, satisfies Definition~\ref{def:stt}.

From \eqref{eq:ball}, for every $t \in [0, t_c]$, there exists a sampled point $t_s$ such that $|t - t_s| \leq \epsilon$. Therefore, for all $t \in [0, t_c]$:
\begin{align*}
    \text{(i) } &-\rad(q_\rad^*, t) + \rad_d = -\rad(q_\rad^*, t_s) + \rad_d + \rad(q_\rad^*, t_s) - \rad(q_\rad^*, t) \\
    &\leq \ \eta^* + \mathcal{L}_\rad |t-t_s| \leq \eta^* + \mathcal{L}_\rad\varepsilon \leq \eta^* + \mathcal{L}\epsilon \leq 0.
\end{align*}
This implies that $\rad(q_\rad^*, t) \geq \rad_d > 0$ for all $t \in [0, t_c]$.
\begin{align*}
     \text{(ii) } &\| \cen(q_{\cen},t) - \cen_{\X} \| + \rad(c_\rad,t) - \rad_{\X} \\
     = &\| \cen(q_{\cen},t) - \cen_{\X} \| - \| \cen(q_{\cen},t_s) - \cen_{\X} \| + \rad(c_\rad,t) - \rad(c_\rad,t_s) 
     +\| \cen(q_{\cen},t_s) - \cen_{\X} \| + \rad(c_\rad,t_s) - \rad_{\X} \\
     \leq & \mathcal{L}_\cen |t-t_s| + \mathcal{L}_\rad |t-t_s| + \eta^* \leq \eta^* + \mathcal{L}\epsilon \leq 0.
\end{align*}
This implies that $\Gamma(q_\cen^*, q_\rad^*, t) \subseteq \X$ for all $t \in [0, t_c]$.
\begin{align*}
    \text{(iii) } &- \text{dist}(\cen(q_\cen,t),\U(t)) + \rad(q_\rad,t) \\ 
    = &- \text{dist}(\cen(q_\cen,t),\U(t)) + \text{dist}(\cen(q_\cen,t_s),\U(t))\\
    &-\text{dist}(\cen(q_\cen,t_s),\U(t)) + \text{dist}(\cen(q_\cen,t_s),\U(t_s))\\
    &+ \rad(q_\rad,t) - \rad(q_\rad,t_s) - \text{dist}(\cen(q_\cen,t_s),\U(t_s)) + \rad(q_\rad,t_s) \\
    \leq &\|\cen(q_\cen,t) - \cen(q_\cen,t_s)\| 
    + \mathcal{L}_U|t-t_s| + \|\rad(q_\rad,t) - \rad(q_\rad,t_s)\| 
    + \eta_S^* 
    \leq (\mathcal{L}_{\cen} + \mathcal{L}_U + \mathcal{L}_{\rad})\epsilon \leq \mathcal{L}\epsilon + \eta_S^* \leq 0.
\end{align*}
This implies that $\Gamma(q_\cen^*, q_\rad^*, t) \cap \U = \emptyset$ for all $t \in [0, t_c]$.

Since the starting and ending constraints are enforced directly at $t=0$ and $t=t_c$, all conditions of Definition~\ref{def:stt} are satisfied when condition \eqref{eq:satisfy} is met.
\end{proof}

\begin{remark}
The Lipschitz constants $\mathcal{L}_\rad$ and $\mathcal{L}_\cen$ are needed to verify condition~\eqref{eq:satisfy}. An estimation procedure for these constants is provided in \cite{das2025spatiotemporal}.
\end{remark}

\section{Controller Synthesis}
This section presents the proposed closed-form robust control design procedure. 

\begin{figure}
    \centering
    \includegraphics[width=0.4\linewidth]{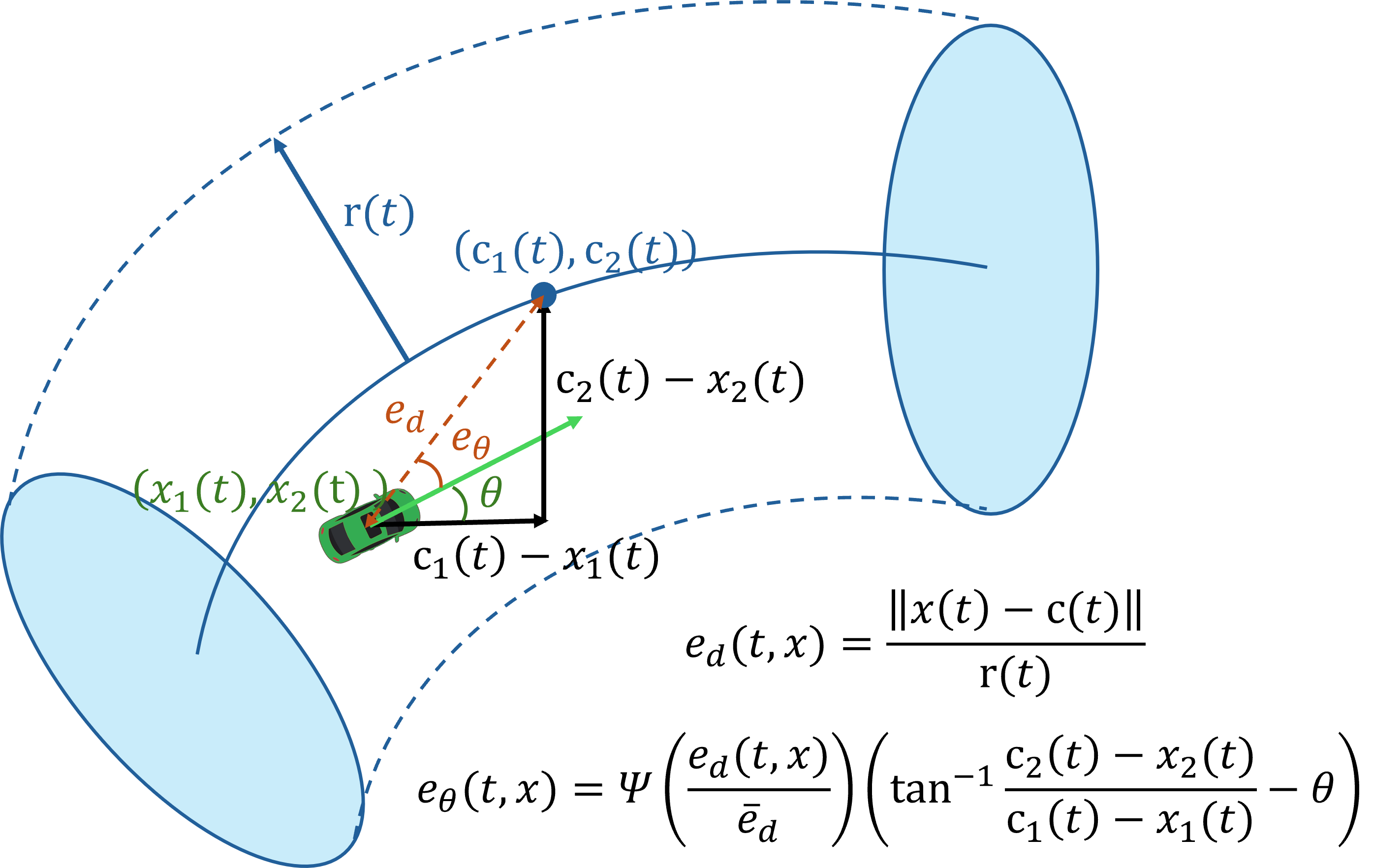}
    \caption{Robot inside circular STT.}
    \label{fig:tube}
\end{figure}

Given the circular spatiotemporal tube (STT)
$\Gamma(t)=\B(\cen(t),\rad(t))$ obtained as discussed in Section~\ref{sec:stt}, we define the distance error $e_d$ and the orientation error $e_\theta$:
\begin{align} 
e_d(t,x) &= \frac{\|x(t) - \cen(t)\|}{\rad(t)} \\ 
e_\theta(t,\xi) &= \Psi\left(\frac{e_d(t,x)}{\overline{e}_d}\right)\frac{2}{\pi} \left( \tan^{-1} \left( \frac{\cen_2(t) - x_2(t)}{\cen_1(t) - x_1(t)} \right) - \theta \right). \notag
\end{align}
Here, $\overline{e}_d \in (0,1)$ is a design threshold and $\Psi:[0,\infty)\to[0,1]$ is the smooth activation function, defined as:
\begin{equation} 
\Psi(s) = 
    \begin{cases} 
        0, &s \in [0,1-\Delta] \\
        1, &s \in [1, \infty] 
    \end{cases} 
\end{equation}
with a smooth transition for $s \in (1-\Delta,1)$, where $\Delta \in (0,1)$. Note that near the tube centre, $e_d(t,x)/\overline{e}_d$ is small and $\Psi\big(e_d(t,x)/\overline{e}_d\big)$ is 0. This factor helps remove the singularities that may arise when $e_d$ goes to 0.

Our goal is to ensure that these errors remain within $(-1, 1)$ for all $t \in \R_0^+$ and eventually converge to $0$ as $t \rightarrow \infty$. To enforce this, we introduce exponentially decaying funnel functions that bound the error dynamics over time:
\begin{align*}
    \rho_d(t) &= (\rho_{d,0}-\rho_{d,\infty})\exp (-l_dt) + \rho_{d,\infty}, \quad
    \rho_\theta(t) = (\rho_{\theta,0}-\rho_{\theta,\infty})\exp (-l_\theta t) + \rho_{\theta,\infty},    
\end{align*}
where $\rho_{d,0}, \rho_{\theta,0} \in (0,1)$ define the initial funnel widths, $\rho_{d,\infty} \in (0,\rho_{d,0}), \rho_{\theta,\infty} \in (0,\rho_{\theta,0})$ specify the final widths, and $l_d, l_\infty \in \R_0^+$ a the exponential decay rates. 

Given the funnel functions, we then define normalized error $\hat{e}=[\hat{e}_d, \hat{e}_\theta]^\top$ and transformed error $\varepsilon=[\varepsilon_d, \varepsilon_\theta]^\top$ as follows:
\begin{align*}
    \hat{e}_{d}(t,\xi) &= \frac{e_d(t,x)}{\rho_d(t)}, 
    &&\varepsilon_{d}(t,x) = \log \left( \frac{1 + \hat{e}_d(t,\xi)}{1 - \hat{e}_d(t,\xi)} \right), \\
    \hat{e}_{\theta}(t,\xi) &= \frac{e_\theta(t,\xi)}{\rho_\theta(t)},
     &&\varepsilon_{\theta}(t,\xi) = \log \left( \frac{1 + \hat{e}_\theta(t,\xi)}{1 - \hat{e}_\theta(t,\xi)} \right).
\end{align*}
For conciseness, we omit the explicit dependence on $(t,x)$ or $(t,\xi)$ in the following analysis. 

The following theorem presents the proposed control law that ensures the system remains within the STT.

\begin{figure*}[h!]
    \centering
    \includegraphics[width=0.9\textwidth]{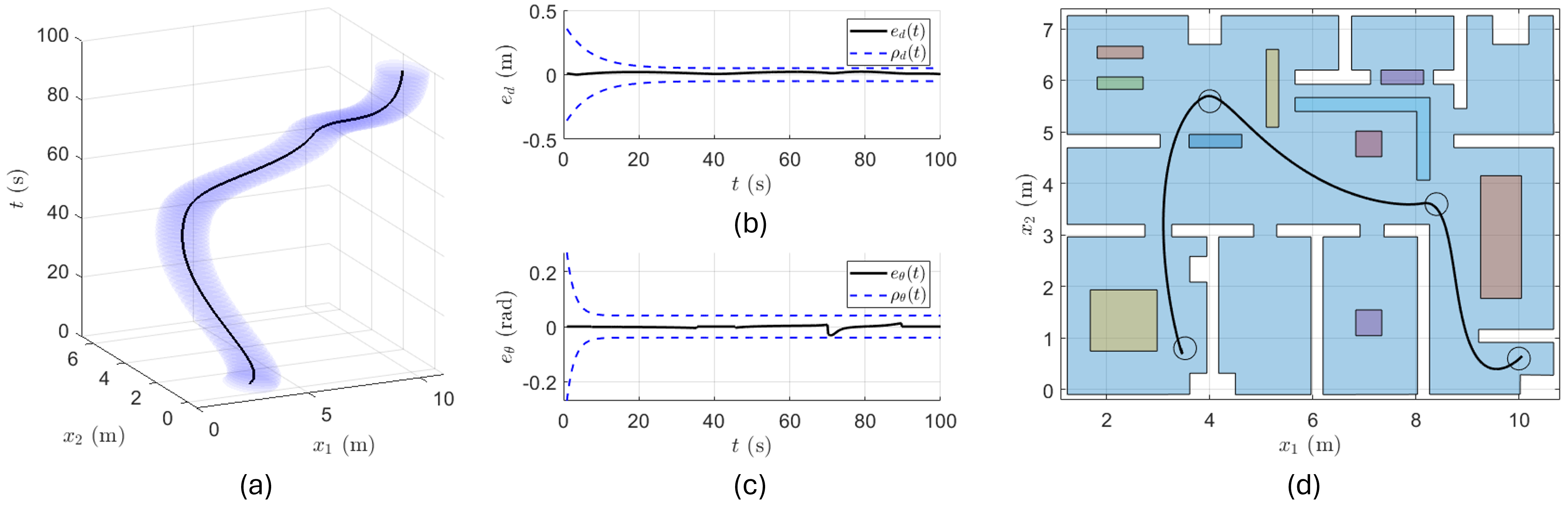}
    \caption{The constructed STT and the corresponding vehicle trajectory navigating through an office space.}
    \label{fig:sim1}
\end{figure*}

\begin{theorem}\label{thm:control}
    Consider the dynamical system in \eqref{eqn:sysdyn_unicycle}, with a T-RAS task as defined in Definition~\ref{def:ptras}. Given the STT, derived in Section~\ref{sec:stt}, if the initial state lies within the STT, then the closed-form control law:
    \begin{align}\label{eqn:Control}
        v(t,\xi) &= k_d \left( \varepsilon_d \alpha_d \cos ( \psi-\theta ) - \varepsilon_\theta \alpha_\theta \sin( \psi-\theta ) \right), \notag \\
        \omega(t,\xi) &= k_\theta \varepsilon_\theta \alpha_\theta,
    \end{align}
    with $ k_d, k_\theta>0$, guarantee that the system state remains inside the STT for all time, thus satisfying the T-RAS task.
    Here, $\alpha_d = \frac{2}{(1-\hat{e}_d^2)} \frac{1}{\rho_d  \rad}$, $\alpha_\theta = \frac{2}{(1-\hat{e}_\theta^2)} \frac{2}{\pi} \frac{1}{\rho_\theta} \frac{1}{e_d \rad}$ and $\psi = \tan^{-1}\left( \frac{\cen_2 - x_2}{\cen_1 - x_1} \right)$.
\end{theorem}

\begin{proof}
The proof proceeds by considering two separate cases, (i) $e_d(t,x)<\overline{e}_d$ and (ii) $e_d(t,x)\geq\overline{e}_d$.

\textbf{(i) Case 1: $e_d(t,x)<\overline{e}_d$.}  \\
In this case,
$e_d(t,x)<\overline{e}_d\leq 1,$
which directly implies $\|x(t)-\cen(t)\| < \rad(t)$. Thus, the system is already within the tube.

\textbf{(ii) Case 2: $e_d(t,x)\geq\overline{e}_d$.} \\
Here, $\Psi(e_d(t,x) / \overline{e}_d) = 1$, and thus,
$e_\theta(t,\xi) = \frac{2}{\pi} \left( \tan^{-1} \left( \frac{\cen_2(t) - x_2(t)}{\cen_1(t) - x_1(t)} \right) - \theta \right).$
Differentiating the normalized error $\hat{e}(t) = [\hat{e}_d(t), \hat{e}_\theta (t)]^\top$ along the system trajectory, we get
\begin{align*}
    \dot{\hat{e}} =  
        \Big[(\dot{e}_d-\hat{e}_d \dot{\rho}_d)\frac{1}{\rho_d}, \ (\dot{e}_\theta-\hat{e}_\theta \dot{\rho}_\theta)\frac{1}{\rho_\theta} \Big]^\top
    := h(t,\hat{e}),
\end{align*}
where 
$\dot{e}_d = \frac{(x_1 - \cen_1)(v \cos \theta + d_1 - \dot{\cen}_1) + (x_2 - \cen_2)(v \sin \theta + d_2 - \dot{\cen}_2)}{e_d \rad^2} - \frac{e_d \dot{\rad}}{\rad}, \\
\dot{e}_\theta = \frac{2}{\pi}\Big( \frac{(x_1 - \cen_1)(v \sin \theta + d_2 - \dot{\cen}_2) - (x_2 - \cen_2)(v \cos \theta + d_1 - \dot{\cen}_1)}{e_d^2 \rad^2} - (\omega + d_\theta) \Big).$
We also define the constraints for $\hat{e}$ through the open and bounded set $\mathbb{D} :=(-1, 1)^2$. 

We now prove the theorem in three steps. First, we establish the existence of a maximal solution for the normalized error vector $\hat{e}$, defined on the interval $[0,\tau_{\max}]$ such that $\hat{e}(t) \in \mathbb{D} := (-1,1)^2$, for all $t \in [0, \tau_{\max})$. This ensures that the solution remains within the domain $\mathbb{D}$ throughout the maximal interval. Second, we demonstrate that the proposed control law \eqref{eqn:Control} guarantees that $\hat{e}(t)$ is confined to a compact subset of $\mathbb{D}$. Finally, we show that the maximal existence time $\tau_{\max}$ can be extended to infinity, completing the proof. 

\textbf{Step 1: Existence of a Maximal Solution.}\\ 
Since the initial position $x(0)$ satisfies $\| x(0) - \cen(0) \| \leq \rad(0)$, the initial normalized error $\hat e(0)$ lies within the constrained region $\mathbb{D}$. Further, the STT functions $\cen(t)$ and $\rad(t)$ are bounded and continuously differentiable, the trigonometric functions are locally Lipschitz, and the control law $[v, \omega]^\top$ is smooth over $\mathbb{D}$. Therefore, $h(t,\hat{e})$ is bounded and continuously differentiable in $t$ and locally Lipschitz in $\hat{e}$ over $\mathbb{D}$. 
According to \cite[Theorem 54]{sontag}, this guarantees a unique maximal solution to the initial value problem $\dot{\hat{e}} = h(t,\hat{e})$ on $[0, \tau_{\max})$, where $\hat{e}(t) \in \mathbb{D}$.

\textbf{Step 2: Confinement of the Normalized Error.}\\
Consider the positive definite and radially unbounded Lyapunov function
$$ V = \frac{1}{2}(\varepsilon_d^2 + \varepsilon_\theta^2). $$
Differentiating $V$ along the system trajectories:
\begin{align*}
    \dot{V} &= \varepsilon_d \dot{\varepsilon}_d + \varepsilon_\theta \dot{\varepsilon}_\theta \\
    &= -(\varepsilon_d \alpha_d \cos(\psi - \theta) - \varepsilon_\theta \alpha_\theta \sin(\psi - \theta)) v
    - ( \varepsilon_\theta \alpha_\theta e_d \rad) \omega 
        + \varepsilon_d \alpha_d \phi_d + \varepsilon_\theta \alpha_\theta \phi_\theta,
\end{align*}
where $\phi_d = \frac{ (x_1 - \cen_1)(d_1-\dot{\cen}_1) + (x_2 - \cen_2)(d_2-\dot{\cen}_2)}{e_d \rad} - e_d \dot{\rad} - \rad \hat{e}_d \dot{\rho}_d$ and $\phi_\theta = \frac{(x_1 - \cen_1) (d_2-\dot{\cen}_2) - (x_2 - \cen_2)(d_1-\dot{\cen}_1)}{e_d \rad} - e_d \rad d_\theta - \frac{\pi}{2} e_d \rad \hat{e}_\theta \dot{\rho}_\theta$.

Substituting the control laws:
\begin{align*}
    \dot{V} &= -k_d(\varepsilon_d \alpha_d \cos(\psi - \theta) - \varepsilon_\theta \alpha_\theta \sin(\psi - \theta))^2
     - k_\theta( \varepsilon_\theta \alpha_\theta)^2 e_d \rad
        + \varepsilon_d \alpha_d \phi_d + \varepsilon_\theta \alpha_\theta \phi_\theta.
\end{align*}

To establish a strict upper bound on $\dot{V}$, we utilize the algebraic inequality $-(A - B)^2 \leq -\frac{1}{2}A^2 + B^2$. Applying this to the linear velocity control term yields:
\begin{align*}
    \dot{V} &\leq -\frac{1}{2} k_d \varepsilon_d^2 \alpha_d^2 \cos^2(\psi - \theta) + k_d \varepsilon_\theta^2 \alpha_\theta^2 \sin^2(\psi - \theta) 
    - k_\theta \varepsilon_\theta^2 \alpha_\theta^2 e_d \rad + |\varepsilon_d| \alpha_d |\phi_d| + |\varepsilon_\theta| \alpha_\theta |\phi_\theta| \\
    &= -\frac{1}{2} k_d \cos^2(\psi - \theta) \alpha_d^2 \varepsilon_d^2
    - \big( k_\theta e_d \rad - k_d \sin^2(\psi - \theta) \big) \alpha_\theta^2 \varepsilon_\theta^2 
    + |\varepsilon_d| \alpha_d |\phi_d| + |\varepsilon_\theta| \alpha_\theta |\phi_\theta|.
\end{align*}

A crucial property of the proposed orientation funnel is that $\hat{e}_\theta \in (-1, 1)$ guarantees $e_\theta \in (-\rho_{\theta,0}, \rho_{\theta,0})$. From the definition of $e_\theta$, this implies $|\psi - \theta| \leq \frac{\pi}{2} \rho_{\theta,0}$. Since the initial funnel width $\rho_{\theta,0} < 1$, the angular error is bounded away from $\pm \pi/2$. Thus, $\cos^2(\psi - \theta) \geq \cos^2(\frac{\pi}{2} \rho_{\theta,0}) := c_{\min} > 0$. 

Furthermore, since we are analyzing Case 2 where $e_d \geq \bar{e}_d$, we have $e_d \rad \geq \bar{e}_d \rad_{\min}$. By selecting the orientation control gain such that $k_\theta > k_d / (\bar{e}_d \rad_{\min})$, we ensure the coefficient of $\varepsilon_\theta^2$ is strictly positive. Let $\lambda_d = \frac{1}{2} k_d c_{\min}$ and $\lambda_\theta = k_\theta \bar{e}_d \rad_{\min} - k_d > 0$. The derivative of the Lyapunov function is bounded by:
\begin{align*}
    \dot{V} \leq -\lambda_d \alpha_d^2 \varepsilon_d^2 - \lambda_\theta \alpha_\theta^2 \varepsilon_\theta^2 + |\varepsilon_d \alpha_d| |\phi_d| + |\varepsilon_\theta \alpha_\theta| |\phi_\theta|.
\end{align*}

Applying Young's Inequality ($|xy| \leq \frac{\Theta}{2} x^2 + \frac{1}{2\Theta} y^2$) to the remainder terms with $\Theta = \lambda_d$ and $\Theta = \lambda_\theta$ respectively:
\begin{align*}
    \dot{V} &\leq -\frac{\lambda_d}{2} \alpha_d^2 \varepsilon_d^2 - \frac{\lambda_\theta}{2} \alpha_\theta^2 \varepsilon_\theta^2 + \frac{\phi_d^2}{2\lambda_d} + \frac{\phi_\theta^2}{2\lambda_\theta}.
\end{align*}

{From Step~1, since $\hat{e}_d, \hat{e}_\theta \in (-1,1)$, the state $\xi(t)$ is constrained by bounds independent of $\tau_{\max}$. Since the tube parameters $\cen(t), \rad(t)$, the funnel signals $\rho_d(t), \rho_\theta(t)$, and the disturbances $d(\xi,t)$ are uniformly bounded, the remainder terms $\phi_d$ and $\phi_\theta$ are uniformly bounded by a finite constant $C_\phi$. Let $D_\phi = C_\phi^2 / (2\lambda_d) + C_\phi^2 / (2\lambda_\theta)$. 

Since $\alpha_d$ and $\alpha_\theta$ contain the terms $(1-\hat{e}_d^2)^{-1}$ and $(1-\hat{e}_\theta^2)^{-1}$ respectively, they are bounded from below by positive constants $\underline{\alpha}_d, \underline{\alpha}_\theta > 0$. Therefore:
\begin{align*}
    \dot{V} \leq -\frac{\lambda_d \underline{\alpha}_d^2}{2} \varepsilon_d^2 - \frac{\lambda_\theta \underline{\alpha}_\theta^2}{2} \varepsilon_\theta^2 + D_\phi.
\end{align*}
}
This implies that $\dot{V} < 0$ whenever $\|\varepsilon\|$ is sufficiently large. Specifically, there exists a time-independent upper bound $\varepsilon^* \in \mathbb{R}_0^+$ such that $\|\varepsilon(t)\| \leq \varepsilon^*, \ \forall t \in [0, \tau_{\max})$.

Taking the inverse of the transformed error bounds, we get:
\begin{align*}
    -1 < \frac{\exp(-\varepsilon^*)-1}{\exp(-\varepsilon^*)+1} =: e_{i,L} \leq \hat{e}_i \leq e_{i,U} := \frac{\exp(\varepsilon^*)-1}{\exp(\varepsilon^*)+1} < 1,
\end{align*}
$\forall t \in [0, \tau_{\max}), \ \text{for } i \in \{d, \theta\}.$ Therefore, by employing the control law (\ref{eqn:Control}), we can constrain $\hat{e}$ to a compact subset of $\mathbb{D}$ as:
    $\hat{e}(t) \in [\hat{e}_L, \hat{e}_U] =: \mathbb{D}' \subset \mathbb{D}, \forall t \in [0, \tau_{\max}) $, where $\hat{e}_L = [e_{d,L}, e_{\theta,L}]^{\top}$ and $\hat{e}_U = [e_{d,U}, e_{\theta,U}]^{\top}$.

\textbf{Step 3. Extension of $\tau_{\max}$ to $\infty$.}\\
We know that $e(t) \in \mathbb{D}', \forall t \in [0, \tau_{\max})$, where $\mathbb{D}'$ is a non-empty compact subset of $\mathbb{D}$.
However, if $\tau_{\max} < \infty$ then according to \cite[Proposition C.3.6]{sontag}, $\exists t' \in [0, \tau_{\max})$ such that $e(t) \notin \mathbb{D}$. This leads to a contradiction!
Hence, we conclude that $\tau_{\max}$ can be extended to $\infty$.

In conclusion, the control strategy \eqref{eqn:Control} guarantees that both the distance and orientation errors evolve within their respective funnels and eventually decay down to 0. Consequently, the vehicle $x(t)$ is ``pulled" toward the centre of the STT while always staying within the STT, thereby satisfying the T-RAS specification.
\end{proof}

\begin{remark}
The time-varying control law in \eqref{eqn:Control} provides a closed-form solution that guarantees fulfillment of the T-RAS task, even in the presence of unknown disturbances in control-affine systems. Moreover, the proposed algorithm can be extended to account for high-order dynamics by following the backstepping-like approach outlined in \cite{verginis2022kdf}. 
\end{remark}

\begin{figure*}[h!]
    \centering
    \includegraphics[width=0.95\textwidth]{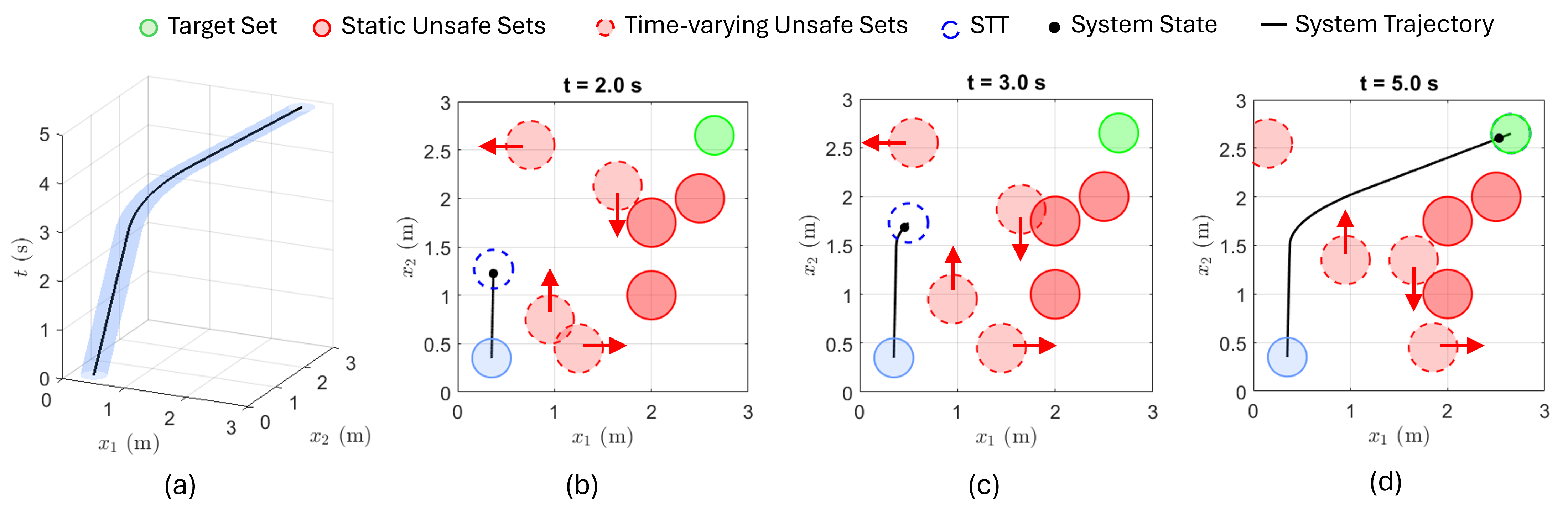}
    \caption{The constructed STT and the corresponding vehicle trajectory in a dynamic environment with time-varying obstacles.}
    \label{fig:dyn}
\end{figure*}

\section{Case Studies}

To demonstrate the effectiveness of the proposed approach, we present two case studies. All the computations were performed using Python 3.10 on a machine with a Windows operating system with an Intel Core i7-14700 CPU, 32 GB RAM. We have used the Z3 solver \cite{de2008z3} to solve the optimization programs. The implementation of the proposed method can be found at: \hyperlink{}{https://github.com/FocasLab/DifferentialDriveSTT}.

\subsection{Navigation in a Cluttered Office-Like Environment}
In the first scenario, the robot has to navigate through a cluttered office-like environment with multiple static obstacles. The task is defined as a sequential reach-avoid objectives. For a $100$ second mission, the robot has to start within the ball $\So:=\mathcal{B}([3.5,0.8], 0.2)$ and should reach its first target $\T_1:=\mathcal{B}([4, 5.6], 0.2)$ within $40$ seconds, while the next target $\T_2:=\mathcal{B}([8.4, 3.6], 0.2)$ in the next 30 seconds and finally reach its goal $\T_3:=\mathcal{B}([10, 0.6], 0.2)$ within the prescribed time while avoiding all the static obstacles in the environment. 
Given the sequential reach-avoid-stay tasks, we choose the center vector of STTs for individual tasks as two polynomials in time $t$ with the basis functions being the monomials of $t$ such as $\{1, t, t^2, \ldots\}$ for two dimensions. We consider piecewise polynomial tubes and consider the degree of polynomials to be $4,6,4$ for the first, second and third mission respectively. Considering the obstacles to be static, $\mathcal{L}_U = 0$. The algorithm converges with $\eta = -0.1$, with $\epsilon = 0.25$. The obtained Lipschitz constant of the STT is $\mathcal{L}_c = 0.2316, \mathcal{L}_r = 0$. Therefore, $\eta^* + \mathcal{L}\epsilon = -0.1 + 0.2316 \times 0.25 = -0.0421 < 0$, satisfying the condition of Theorem \ref{th:constr}, thereby implying the STT is formally verified.  

The synthesized spatiotemporal tube (STT) is shown in Fig.~\ref{fig:sim1}(a), and the corresponding robot trajectory under the proposed control law is shown in Fig.~\ref{fig:sim1}(d). The distance and orientation errors, $e_d$ and $e_\theta$, along with their funnel bounds $\rho_d(t)$ and $\rho_\theta(t)$, are plotted in Fig.~\ref{fig:sim1}(b)--(c). Both errors remain strictly within their funnels, confirming that the robot safely and successfully completes the sequence of reach-avoid-stay tasks within the prescribed time. The offline tube synthesis time is $2.081$ seconds for all the missions and per-step online controller synthesis time is $6.34$ $\mu$sec.

\subsection{
Reach-Avoid Task with Dynamic Unsafe Set}
The second scenario involves a dynamic environment, where the robot must reach the target region within a prescribed time of $5$ seconds while avoiding both static and moving obstacles.The velocity of the dynamic obstacles is known a priori which essentially is the Lipschitz constant $\mathcal{L}_U = 0.26$. Here also we consider the center of the STT to be two polynomials of $t$ corresponding to the $x$ and $y$ dimensions. We generate piecewise polynomial tubes, each of which are of degree $2$. With the choice of $\epsilon = 0.005$ and $\eta^* = -0.0075$, the algorithm converges. The Lipschitz constant corresponding to the center of the STT is obtained as $\mathcal{L}_c = 1.21, \mathcal{L}_r = 0$. Therefore the overall Lipschitz constant becomes $\mathcal{L} = 1.47$. Therefore, $\eta^* + \mathcal{L}\epsilon = -0.0075 + 1.47 \times 0.005 = -0.00015 < 0$, satisfying the condition of Theorem \ref{th:constr}, thereby implying the STT is formally verified.
Figure~\ref{fig:dyn}(a) shows the synthesized STT, and Figs.~\ref{fig:dyn}(b)--(d) illustrate the workspace at three different time instants. The results demonstrate that the proposed controller effectively adapts to a changing environment by reshaping the tube, enabling the robot to avoid time-varying obstacles and reach the target within the specified time. The offline tube synthesis time requires an overall time of $2.415$ seconds while the online per step control synthesis time is $2.18$ $\mu$sec.

\begin{figure}
    \centering
    \includegraphics[width=0.45\linewidth]{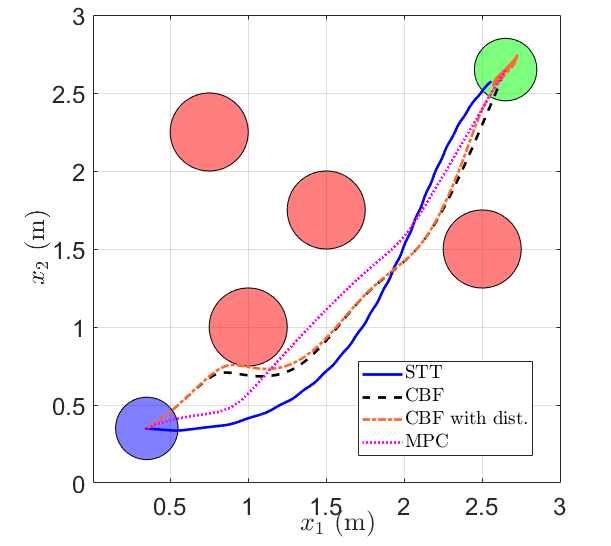}
    \caption{Comparison with existing approaches.}
    \label{fig:comparison}
\end{figure}

\textbf{Comparison:} We compare the proposed STT-based control with state-of-the-art approaches such as CBF \cite{toulkani2022reactive} and MPC \cite{sun2017disturbance}. As shown in Fig.~\ref{fig:comparison}, all methods guide the robot from the same initial region to the target while avoiding obstacles. However, the proposed STT-based approach achieves this with significantly lower per-step online control-synthesis time compared to CBF and MPC. Moreover, under external disturbances considered to be sinusoidal with the magnitude ranging between $0.0001$ to $0.0125$, the CBF and MPC based controller fails in some scenarios to prevent the robot from entering unsafe regions, whereas the proposed STT-based controller keeps the trajectory strictly within the tube, ensuring robust and safe navigation throughout the task. The average online computation time per step and success rate of the algorithms with and without disturbances are presented in the Table \ref{tab:comparison}.

\begin{table*}[]
\centering
\caption{Comparison Table of STT with classical algorithms}
\label{tab:comparison}
\resizebox{0.85\textwidth}{!}{
\begin{tabular}{|c|cc|cc|}
\hline
\multirow{2}{*}{} & \multicolumn{2}{c|}{Average Computation Time per Step ($\mu$s)} & \multicolumn{2}{c|}{Success Rate (\%)}     \\ \cline{2-5} 
& \multicolumn{1}{c|}{Without Disturbance} & With Disturbance & \multicolumn{1}{c|}{Without Disturbance} & With Disturbance \\ \hline
STT & \multicolumn{1}{c|}{$1.1 \pm 0.0$} & $1.1 \pm 0.0$ & \multicolumn{1}{c|}{$100.00$} & $100.00$ \\ \hline
CBF \cite{toulkani2022reactive} & \multicolumn{1}{c|}{$395.34 \pm 43.8$} & $399.1 \pm 19.2$ & \multicolumn{1}{c|}{$100.00$} & $90.00$ \\ \hline
MPC \cite{sun2017disturbance} & \multicolumn{1}{c|}{$63583.62 \pm 1505.51$} & $63500.64 \pm 1146.84$ & \multicolumn{1}{c|}{$100.00$} & $44.00$ \\ \hline
\end{tabular}
}
\end{table*}

\section{Conclusion and Future Work}
In this work, we address the temporal reach-avoid-stay (T-RAS) problem for differential-drive robots. A spatiotemporal tube (STT) with circular cross-sections is constructed using a sampling-based approach over a predefined time horizon, providing a safe corridor that connects the start and target regions while avoiding time-varying obstacles. A closed-form control law is then derived to keep the robot trajectory within the STT, ensuring T-RAS satisfaction. The proposed controller is computationally efficient and robust to disturbances, outperforming control barrier function and model predictive control methods in both robustness and computation time. Future work will focus on incorporating explicit input constraints in tube design.

\bibliographystyle{unsrt} 
\bibliography{sources} 

\end{document}